\documentclass{article}
\usepackage{ijcai21}

\RequirePackage[T1]{fontenc} 
\RequirePackage[tt=false, type1=true]{libertine} 
\RequirePackage[varqu]{zi4} 

\RequirePackage[libertine]{newtxmath}
\usepackage[letterspace=-25]{microtype}

\usepackage{soul}
\usepackage{url}
\usepackage[hidelinks]{hyperref}
\usepackage[utf8]{inputenc}
\usepackage[small]{caption}
\usepackage{graphicx}
\usepackage{amsmath}
\usepackage{booktabs}
\usepackage{savesym}
\savesymbol{Bbbk}
\savesymbol{openbox}

\usepackage{amssymb}
\usepackage{amsthm}
\usepackage{MnSymbol}
\usepackage{color}
\usepackage{url}
\usepackage{multirow}
\usepackage{xcolor}
\usepackage{dsfont}

\restoresymbol{NEW}{Bbbk}
\restoresymbol{NEW}{openbox}

\usepackage{savesym}
\savesymbol{Bbbk}
\savesymbol{openbox}

\usepackage{xcolor}
\usepackage{url}
\usepackage[utf8]{inputenc}
\usepackage{graphicx}
\usepackage{amsmath}
\usepackage{amssymb}
\usepackage{amsthm}
\usepackage{MnSymbol}
\usepackage{booktabs}
\usepackage{algorithm}
\usepackage{algorithmic}
\restoresymbol{NEW}{Bbbk}
\restoresymbol{NEW}{openbox}
\usepackage{tcolorbox}

\def\resp{respectively}

\def\Models{\ensuremath{\mathcal{M}}}
\def\model{\ensuremath{\mathcal{M}}}
\def\CFXs{\ensuremath{\mathcal{C}}}
\def\Classes{\ensuremath{\mathcal{L}}}
\def\Props{\ensuremath{\mathcal{P}}}
\def\inn{\ensuremath{\mathbf{x}}}

\def\ProblemFull{\textbf{Recourse-Aware Ensembling}}
\def\Problem{RAE}
\def\AggModels{\ensuremath{\mathcal{M}^{n}}}
\def\method{naive ensembling}

\def\Args{\ensuremath{\mathcal{X}}}

\def\Atts{\ensuremath{\mathcal{A}}}
\def\Supps{\ensuremath{\mathcal{S}}}

\def\arg{\ensuremath{\alpha}}
\def\arga{\ensuremath{\arg_1}}
\def\argb{\ensuremath{\arg_2}}
\def\argc{\ensuremath{\arg_3}}

\def\model{\ensuremath{M}}

\def\cfx{\ensuremath{\mathbf{c}}}

\def\prop{\ensuremath{\pi}}

\def\ppreceq{\preceq_\Props}
\def\psucc{\succ_\Props}

\def\psimeq{\simeq_\Props}

\def\mpreceq{\preceq_\Models}
\def\msucc{\succ_\Models}
\def\msucceq{\succeq_\Models}
\def\msimeq{\simeq_\Models}

\def\class{\ensuremath{\ell}}

\def\rel{\ensuremath{r}}

\newcommand\myprob[3]{%
  \begin{tcolorbox}
  {\bfseries Problem: #1}\\
  {\bfseries Input}: #2\\
  {\bfseries Output}: #3\par
  \end{tcolorbox}
}

\usepackage[switch]{lineno}

\usepackage{tablefootnote}

\definecolor{cadmiumgreen}{rgb}{0.0, 0.75, 0.40}
\definecolor{richcarmine}{rgb}{0.85,0.00,0.23}

\def\resp{resp.}

\newtheorem{example}{Example}
\newtheorem{theorem}{Theorem}
\newtheorem{lemma}{Lemma}
\newtheorem{definition}{Definition}
\newtheorem{proposition}{Proposition}

\usepackage{fancyhdr}
\fancypagestyle{firstpage}{%
    \fancyhf{}
    \setlength{\headsep}{0.3in}
    \cfoot{\thepage}
}
\fancypagestyle{otherpages}{%
    \fancyhf{}
    \setlength{\headsep}{0.3in}
    \rhead{Paper Title}
    \lhead{Authors}
    \cfoot{\thepage}
}

\title{Recourse under Model Multiplicity via Argumentative Ensembling \\(Technical Report)}

\author {
    Junqi Jiang\footnote{Equal contribution.},
    Antonio Rago\footnotemark[1],
    Francesco Leofante,
    Francesca Toni\\
\affiliations\large
    Department of Computing, Imperial College London, UK\\
\emails
    \{junqi.jiang, a.rago, f.leofante, f.toni\}@imperial.ac.uk
}
\begin{document}
\maketitle

\begin{abstract}
Model Multiplicity (MM) 
arises when
multiple, equally performing machine learning models can be trained to solve the same prediction task. Recent studies 
show that models obtained under MM may produce inconsistent predictions for the same input. {When this occurs}
, it  becomes challenging to 
provide counterfactual explanations (CEs), a common
means for offering  
recourse recommendations to individuals 
negatively affected by models' predictions
. 
%
In this paper, we formalise this problem, which we name \emph{recourse-aware ensembling}, and identify several desirable properties which methods for solving it should satisfy.
We 
show 
that existing ensembling methods, naturally extended in different ways to provide CEs, fail to satisfy these properties.
We then introduce \emph{argumentative ensembling}, deploying computational argumentation 
to 
guarantee robustness of CEs to MM, while also accommodating customisable user preferences.
We show theoretically and experimentally that argumentative ensembling 
satisfies properties which the existing methods lack, and that the trade-offs are minimal wrt 
accuracy.

\end{abstract}

\section{Introduction}
\label{sec:intro}

Model Multiplicity (MM), also known as predictive multiplicity or the Rashomon Effect, refers to a scenario where multiple, equally performing machine learning (ML) models may be trained to solve a prediction task \cite{Black_22,breiman2001statistical,Marx_20}. While the existence of multiple models that achieve the same accuracy is not a problem per se, recent literature~\cite{Marx_20,Black_22} has drawn attention to the fact that these models may differ greatly in their internals 
and might thus produce inconsistent predictions when deployed. Consider the commonly used scenario of a loan application, where an individual modelled by input $\inn$ with features \textit{unemployed} status, \textit{33} years of age and \textit{low} credit rating applies for a loan. Assume the bank has trained a set of ML models $\Models=\{\model_1, \model_2, \model_3\}$ to predict whether the loan should be granted or not. 
Even though each $\model_i$ may exhibit good performance overall
, their internal differences 
may lead to conflicts
, e.g. if $\model_1(\inn) = \model_2(\inn) = 0$ (i.e., reject), while $\model_3(\inn) =1 $ (i.e., accept).

Ensembling techniques are commonly used to deal with MM scenarios
~\cite{black2022selective,Black_22}. A standard 
such technique is \textit{\method}~\cite{black2022selective}, where the predictions of several models are aggregated to produce a single outcome that reflects the opinion of the majority of models. For instance, \method{} applied to our running example would result in a rejection, as a majority of the models agree that the applicant is not creditworthy. While ensembling methods have been shown to be effective in practice, their application to consequential decision-making tasks raises some important challenges. 

Indeed, 
these methods tend to ignore
the need to provide avenues for recourse to users negatively impacted by the models’ outputs, which the ML literature typically achieves via the provision of \emph{counterfactual explanations} (CEs) for the predictions (see \cite{guidotti2022counterfactual,mishra2021survey} for recent overviews). Dealing with MM while also taking CEs into account is non-trivial. Indeed, standard algorithms designed to generate CEs for
single models typically fail to produce recourse recommendations
that are valid across equally performing models~\cite{pawelczyk2022uai,LeofanteBR23}. This phenomenon may have troubling implications as a lack of robustness may lead users to question whether a CE is actually explaining the underlying decision-making task and is not just an artefact of a (subset of) model(s). 

Another challenge is that naive ensembling ignores other meta-evaluation aspects of models like fairness, robustness, and interpretability, while it has been shown that models under MM can demonstrate substantial differences in these regards \cite{coston2021characterizing,rudin2019stop,d2022underspecification} 
 and users may have strong preferences for some of these aspects, e.g. they may prefer model fairness to accuracy.

In this paper, we frame the recourse problem under MM formally and propose different approaches to 
accommodate recourse as well as user preferences
. After covering related work (§\ref{sec:related}) and the necessary preliminaries (§\ref{sec:preliminaries}), we make the following contributions. In §\ref{sec:problem}, we purpose a formalisation of the problem and 
several 
desirable properties for ensembling methods for recourse under MM. We also consider two natural extensions of \method{} to accommodate generation of CEs and show that they may violate some of the properties we define. We then propose \textit{argumentative ensembling} in §\ref{sec:main}, a novel technique rooted in computational argumentation (see \cite{AImagazine17,handbook} for overviews). We show that it is able to solve the recourse problem effectively while also naturally incorporating user preferences, a problem in which the conflicts between different objectives quickly become cognitively intractable
. We present extensive experiments  in §\ref{sec:evaluation},
where we show that our framework always provides valid recourse under MM, showing better evaluations on a number of metrics
and without compromising the individual prediction accuracies. We also demonstrate the usefulness of specifying user preferences in our framework.
We then conclude 
in §\ref{sec:conclusion}. 

This report is an extended version of \cite{aamas24}, including supplementary material. The implementation is publicly available at  \url{https://github.com/junqi-jiang/recourse_under_model_multiplicity}.

\section{Related Work}
\label{sec:related}


{\bf Model Multiplicity.} MM has been shown to affect several 
dimensions of trustworthy ML. In particular, among equally accurate models, there could 
be 
different fairness characteristics \cite{wick2019unlocking,dutta2020there,rodolfa2021empirical,coston2021characterizing}, levels of 
interpretability \cite{chen2018interpretable,rudin2019stop,semenova2022existence}, model robustness evaluations \cite{d2022underspecification} and even inconsistent explanations~\cite{dong2019variable,fisher2019all,mehrer2020individual,blackconsistent,ley2023consistent,marx2023but,LeofanteBR23}. 

Recent attempts have been made to 
address the MM problem. \cite{Black_22} suggested candidate models should be evaluated across additional dimensions other than accuracy (e.g., robustness or fairness evaluation thresholds). They provided a potential solution based on applying meta-rules to filter out undesirable models, then use ensemble methods to aggregate them, or randomly select one of them. Extending these ideas, the selective ensembling method of \cite{black2022selective} embeds statistical testing into the ensembling process, such that when the numbers of candidate models predicting an input to the top two classes are close or equal, which could happen under naive ensembling strategies like majority voting, an abstention signal can be flagged for relevant stakeholders. \cite{xin2022exploring} looked at decision trees and proposed an algorithm to enumerate all models obtainable under MM; \cite{roth2023reconciling} instead proposed a model reconciling procedure to resolve the conflicts between two disagreeing models.  Meanwhile, a number of works \cite{Marx_20,watson2023predictive,hsu2022rashomon,semenova2022existence} propose metrics to quantify the extent of MM in prediction tasks.



{\bf Counterfactual Explanations and MM.} A CE for a prediction of an input by an ML model is typically defined as another data point that minimally modifies the input such that the ML would yield a desired classification~\cite{Tolomei_17,Wachter_17}. CEs are often advocated as a means to provide algorithmic recourse for data subjects, and as such, modern algorithms have been extended to enforce additional desirable properties such as~\textit{actionability}~\cite{ustun2019actionable}, \textit{plausiblity}~\cite{dhurandhar2018explanations} and \textit{diversity}~\cite{mothilal2020explaining}. We refer to \cite{guidotti2022counterfactual} for a recent overview.

More central to the MM problem, \cite{pawelczyk2022uai} pointed out that CEs on data manifold are more likely to be robust under MM than minimum-distance CEs. \cite{LeofanteBR23} proposed an algorithm for a given ensemble of feed-forward neural networks to compute robust CEs that are provably valid for all models in the ensemble. Also related to MM is the line of work that focuses on the robustness of CEs against model changes, e.g., parameter updates due to model retraining on the same or slightly shifted data distribution \cite{upadhyay2021towards,pmlr-v162-dutta22a,blackconsistent,nguyen2022robust,bui2022counterfactual,robovertime,oursaaai23,hamman2023robust,jiang2023provably}. These studies usually aim to generate CEs that are robust across retrained versions of the same model, which is different from MM where several (potentially structurally different) models are targeted together.


\textbf{Computational Argumentation.}
Argumentation, as understood in AI, 
inspired by the seminal 
\cite{Dung_95}, is a set of formalisms for 
dealing with 
conflicting information
, as demonstrated in numerous application areas, e.g. online debate \cite{Cabrio_13}, scheduling \cite{Cyras_19} and judgmental forecasting \cite{Irwin_22}.
There have also been a broad range of works (see \cite{Cyras_21,Vassiliades_21} for overviews) demonstrating argumentation's capability for explaining the outputs of AI models, e.g. neural networks \cite{Potyka_21,Dejl_21}, Bayesian classifiers \cite{Timmer_15} and random forests \cite{Potyka_23}.
To the best of our knowledge the only work in which argumentation has been applied to MM specifically is~\cite{Abchiche-Mimouni_23}, where the introduced method extracts and selects winning rules from an ensemble of classifiers, leveraging on argumentation's strengths in providing explainability for this process. This differs from our method as we consider pre-existing CEs computed for single models, rather than using rules as explanations, and we consider preferences in the ensembling.





\section{Preliminaries}
\label{sec:preliminaries}

Given 
a set of classification \emph{labels} $\Classes$, a \emph{model}  is a mapping $M: \mathbb{R}^n \rightarrow \Classes$; we denote that $\model$ classifies an \emph{input} $\inn \in \mathbb{R}^n$ as $\class 
$ iff  $\model(\inn) = \class$. 
In this paper, we focus on binary classification, i.e. $\Classes = \{ 0, 1 \}$, for simplicity, though our approach is extensible to multi-class problems.
Then, a \emph{counterfactual explanation} (CE) for $\inn$, given $\model$, is some $\cfx \in \mathbb{R}^n \setminus \{ \inn \}$ such that $\model(\cfx) \neq \model(\inn)$, which may be optimised by some distance metric between the inputs.

A \emph{bipolar argumentation framework} (BAF) \cite{Cayrol_05} is a tuple $\langle \Args, \Atts, \Supps \rangle$, where $\Args$ is a set of \emph{arguments}, $\Atts \subseteq \Args \times \Args$ is a directed relation of \emph{direct attack} and $\Supps \subseteq \Args \times \Args$ is a directed relation of \emph{direct support}.
Given a BAF $\langle \Args, \Atts, \Supps \rangle$, for any $\arga \in \Args$, we refer to $\Atts(\arga) = \{ \argb | (\argb, \arga) \in \Atts \}$ as \emph{$\arga$'s direct attackers} and to $\Supps(\arga) = \{ \argb | (\argb, \arga) \in \Supps \}$ as \emph{$\arga$'s direct supporters}.
Then, an \emph{indirect attack} from $\arg_x$ on $\arg_y$ is a sequence $\arg_1, \rel_1, \ldots, \rel_{n-1}, \arg_n$, where $n \geq 3$, $\arg_1 = \arg_x$, $\arg_n = \arg_y$, $\rel_{1} \in \Atts$ and $\rel_i \in \Supps$ $\forall i \in \{ 2, \ldots, n - 1 \}$.
Similarly, a \emph{supported attack} from $\arg_x$ on $\arg_y$ is a sequence $\arg_1, \rel_1, \ldots, \rel_{n-1}, \arg_n$, where $n \geq 3$, $\arg_1 = \arg_x$, $\arg_n = \arg_y$, $\rel_{n-1} \in \Atts$ and $\rel_i \in \Supps$ $\forall i \in \{ 1, \ldots, n - 2 \}$.
Straightforwardly, a supported attack on an argument implies there is a direct attack on an argument.

We will also use notions of acceptability of sets of arguments in BAFs~\cite{Cayrol_05}.
A set of arguments $X \subseteq \Args$, also called an \emph{extension}, is said to \emph{set-attack} any $\arga \in \Args$ iff there exists an attack (whether direct, indirect or supported) from some $\argb \in X$ on $\arga$. 
Meanwhile, $X$ is said to \emph{set-support} any $\arga \in \Args$ iff there exists a direct support from some $\argb \in X$ on $\arga$.\footnote{
In \cite{Cayrol_05}, set-supports are defined 
via 
sequences of supports, 
which we do not use 
here.}
Then, a set $X \subseteq \Args$ \emph{defends} any $\arga \in \Args$ iff $\forall \argb \in \Args$, if $\{ \argb \}$ set-attacks $\arga$ then $\exists \argc \in X$ such that $\{ \argc \}$ set-attacks $\argb$.
Any set $X \subseteq \Args$ is then said to be \emph{conflict-free} iff $\nexists \arga, \argb \in X$ such that $\{ \arga \}$ set-attacks $\argb$, and \emph{safe} iff $\nexists \argc \in \Args$ such that $X$ set-attacks $\argc$ and either: $X$ set-supports $\argc$; or $\argc \in X$. (Note that a safe set is guaranteed to be conflict-free.)
The notion of a set $X \subseteq \Args$ being \emph{d-admissible} (based on \emph{admissibility} in \cite{Dung_95}) is then introduced, requiring that $X$ is conflict-free and defends all of its elements. 
This notion is extended to account for safe sets: $X$ is said to be \emph{s-admissible} iff it is safe and defends all of its elements. 
(Note that an s-admissible set is guaranteed to be d-admissible.)
It is also extended for the notion of \emph{external coherence}~\cite{Cayrol_05}: $X$ is said to be \emph{c-admissible} iff it is conflict-free, closed for $\Supps$ and defends all of its elements.
Finally, a set $X \subseteq \Args$ is said to be \emph{d-preferred} (\resp, \emph{s-preferred}, \emph{c-preferred}) iff it is d-admissible (\resp, s-admissible, c-preferred) and maximal wrt set-inclusion.



\section{Recourse under Model Multiplicity}
\label{sec:problem}
 

As mentioned in 
§\ref{sec:intro}, a common way to deal with MM in practice is to employ ensembling techniques, where the prediction outcomes of several models are aggregated to produce a single outcome. Aggregation can be performed in different ways, as discussed in~\cite{Black_22,black2022selective}. In the following, we formalise a notion of \textit{naive ensembling}, adopted in~\cite{Black_22}, and also known as \textit{majority voting}, which will serve as a baseline for our analysis.\footnote{It should be noted that in \cite{Black_22}, the case where there is no majority is not discussed.}



\begin{definition}
\label{def:naive_ensembling}
    Given an input $\inn$, a set of models $\Models$ and a set of labels $\Classes$, we define the set of \emph{top labels} $\Classes_{max} \subseteq \Classes$ as:
    
    \hspace*{1cm} \(
        \Classes_{max} = argmax_{\class_i \in \Classes} | \{ \model_j \in \Models | \model_j(\inn) = \class_i \} |. \nonumber
   \)
   \\
    Then we then use $\AggModels(\inn) \in \Classes_{max}$ to denote the aggregated classification by \emph{naive ensembling}. 
    In the cases where $| \Classes_{max} | > 1$, we select $\AggModels(\inn)$ from $\Classes_{max}$ randomly
    .
    With an abuse of notation, we also let $\AggModels = \{ \model_j \in \Models | \model_j(\inn) = \AggModels(\inn) \}$ denote the set of models that agree on the aggregated classification.
\end{definition}
%
Coming back to our loan example where $\model_1$ and $\model_2$ reject the loan ($\model_1(\inn) = \model_2(\inn) = 0$) while $\model_3$ accepts it ($\model_3(\inn) = 1$), we obtain $\AggModels(\inn) = 0$ and $\AggModels = \{ \model_1, \model_2 \}$. 
Naive ensembling is known to be an effective strategy to mediate conflicts between models and is routinely used in practical applications. However, in this paper, we take an additional step and aim to generate CEs providing recourse for a user that has been impacted by $\AggModels(\inn)$. Recent work by~\cite{LeofanteBR23} has shown that standard algorithms designed to generate CEs for single models typically fail to produce recourse recommendations that are robust across $\AggModels$. One natural idea to address this would be to extend naive ensembling to account for CEs. Next, we formalise this idea 
in terms of several properties that are important in this setting. We then discuss two concrete methods 
extending naive ensembling and 
analyse them in terms of the properties we define.

\subsection{Problem Statement and Desirable Properties}
\label{ssec:desirable_properties}

Consider a non-empty set of models $\Models \!\!=\!\!\{\model_1, \ldots, \model_m\!\}$ and, for an input $\inn$, 
assume a set 
$\CFXs(\inn)\!=\!\{\cfx_1, \ldots, \cfx_m\}$ where 
each $\cfx_i \in \CFXs(\inn)$ is 
a CE for $\inn$, given $\model_i$. In the rest of the paper, wherever it is clear that we refer to a given $\inn$, we use $\CFXs$ and omit its dependency on $\inn$ for readability.
Our aim is to solve the problem 
outlined below. 

\myprob{\ProblemFull{} (\Problem{})}{input $\inn$, set  $\Models$ of models, set $\CFXs$ of CEs}{``optimal'' set $S \subseteq \Models \cup \CFXs$ of models and CEs.}

 To characterise 
optimality, we propose a number of desirable properties for the outputs of ensembling methods. We refer to these outputs as \emph{solutions} for \Problem{}. The 
most basic 
requirement 
requires  
that both models and CEs in the output are non-empty.
\begin{definition}
\label{def:non-emptiness}
    An ensembling method satisfies \emph{non-emptiness} iff for any given input $\inn$, set  $\Models$ of models and set $\CFXs$ of CEs, any 
    solution 
    $S \subseteq \Models \cup \CFXs$ is such that $S \cap \Models \neq \emptyset$ and $S \cap \CFXs \neq \emptyset$. 
\end{definition}

Specifically, 
non-emptiness ensures that 
the \Problem{} 
method returns 
some models and some CEs. 
We then look to ensure that the ensembling method returns a non-trivial set of models, as formalised 
by the next property.

\begin{definition}\label{def:non-triviality}
    An ensembling method satisfies \emph{non-triviality} iff for any given input $\inn$, set  $\Models$ of models and set $\CFXs$ of CEs, any 
    solution $S \subseteq \Models \cup \CFXs$ is such that $| S \cap \Models | > 1$. 
\end{definition}

Clearly, the 
returned models should not disagree amongst themselves on the classification, which leads to our next requirement.



\begin{definition} \label{def:model_agreement}
    An ensembling method satisfies \emph{model agreement} iff for any given input $\inn$, set  $\Models$ of models and set  $\CFXs$ of CEs, any 
    solution $S \subseteq \Models \cup \CFXs$ is such that $\forall \model_i, \model_j \in S \cap \Models$, $\model_i(\inn) = \model_j(\inn)$. 
\end{definition}


The next property, which itself requires model agreement to be satisfied, checks 
whether the 
set of returned models is among the largest of the agreeing sets of models, a motivating property of naive ensembling.

\begin{definition}\label{def:majority_vote}
    An ensembling method satisfies \emph{majority vote} iff it satisfies model agreement and for any given input $\inn$, set  $\Models$ of models, set $\CFXs$ of CEs  and set $\Classes$ of labels, any 
    solution $S \subseteq \Models \cup \CFXs$ is such that, letting $\class_i = \model_j(\inn)$ $\forall \model_j \in S \cap \Models$, $\nexists \class_k \in \Classes \setminus \{ \class_i \}$ such that $| \{ \model_l \in \Models | \model_l(\inn) = \class_k \} | > | \{ \model_
    {l}\in \Models | \model_
    {l}(\inn) = \class_i \} |$.
\end{definition}

Next, we consider the robustness of recourse. Previous work~\cite{LeofanteBR23} considered a very conservative notion of robustness whereby explanations are required to be valid for all models in $\Models$. While this might be desirable in some cases, we highlight that satisfying this property may not always be feasible in practice. We therefore propose a relaxed notion of robustness, which requires that CEs are valid only for the models that support them.

\begin{definition}\label{def:counterfactual_validity}
    An ensembling method satisfies \emph{counterfactual validity} iff for any given input $\inn$, set  $\Models$ of models and set $\CFXs$ of CEs, any 
    solution $S \subseteq \Models \cup \CFXs$ is such that $\forall \model_i \in S \cap \Models$ and $\forall \cfx_j \in S \cap \CFXs$, $\model_i(\cfx_j) \neq \model_i(\inn)$.
\end{definition}

While counterfactual validity is a fundamental requirement for any sound ensembling method, one also needs to ensure that the solutions it generates are coherent, as formalised below. 

\begin{definition}\label{def:explanation-coherence}
    An ensembling method satisfies \emph{counterfactual coherence} iff for any given input $\inn$, set $\Models$ of models and set  $\CFXs$ of CEs, any 
    solution $S \subseteq \Models \cup \CFXs$
    , where $\Models = \{ \model_1, \ldots, \model_m \}$ and $\CFXs = \{ \cfx_1, \ldots, \cfx_m \}$, 
    is such that $\forall i \in \{ 1, \ldots, m \}$, 
    $\model_i \in S$ iff $\cfx_i \in S$.
\end{definition}

Intuitively coherence requires that \textit{(i)} 
a CE is returned only if it is supported by a model and \textit{(ii)} when the 
CE is chosen, then its corresponding model must be part of the support. This ultimately guarantees strong justification as to why a given recourse is suggested since selected models and their reasoning (represented by their CEs) are assessed in tandem. 

The properties defined above may not be all satisfiable at the same time in practice, therefore heuristic approaches might be needed. Next, we discuss two such heuristics to extend naive ensembling towards solving \Problem.

\subsection{
Extending Naive Ensembling for Recourse}
\label{ssec:naive}

We now present two strategies that leverage naive ensembling to solve \Problem{}. In particular, we use the relationship between models in the ensemble $\AggModels$ and their corresponding CEs as follows.

 \begin{definition}
\label{def:baselinecfxs}
\label{def:cfx_selection_methods}
    Consider an input $\inn$, a set $\Models$ of models and a set  $\CFXs$ of CEs. Let $\AggModels \subseteq \Models$ be the set of models obtained by naive ensembling. We define 
        the set of \emph{naive 
        CEs} 
        as:
        
       \hspace*{2cm}  \(
                    \CFXs^n = \{ \cfx_i \in  \CFXs \mid \model_i \in \AggModels \};
       \)
       \\
        and the set of \emph{valid 
        CEs} as:
        
       \hspace*{0.2cm} \(
            \CFXs^{v} = \{ \cfx_i \in  \CFXs \mid \model_i \in \AggModels 
            \wedge \forall \model_j \in \AggModels, \model_j(\cfx_i) \neq \model_j(\inn) \}.
        \)
        \\
    Then, two possible solutions to \Problem{} are $S^n = \AggModels \cup \CFXs^n$, named \emph{augmented ensembling}, or $S^v = \AggModels \cup \CFXs^{v}$, named \emph{robust ensembling}.
\end{definition}

Intuitively, 
augmented ensembling suggests 
taking all the CEs in $\CFXs$ that correspond to the models 
in $\AggModels$. 
Meanwhile, robust ensembling extends augmented ensembling by enforcing the additional constraint that CEs are selected only if they are valid for all models in $\AggModels$. We now provide an illustrative example to clarify the results produced by the two strategies.

\begin{example}
\label{ex:baselines_example}
    Consider 
    $\Models = \{ \model_1, \model_2, \model_3, \model_4, \model_5 \}$ and an input $\inn$ such that $\model_1(\inn) = \model_2(\inn) = \model_3(\inn) = 0$ and $\model_4(\inn) = \model_5(\inn) = 1$. Let $\CFXs = \{\cfx_1, \cfx_2, \cfx_3, \cfx_4, \cfx_5\}$ be the set of CEs generated for $\inn$, i.e. $\model_1(\cfx_1) = \model_2(\cfx_2) = \model_3(\cfx_3) = 1$, while $\model_4(\cfx_4) = \model_5(\cfx_5) = 0$. Applying naive ensembling to $\Models$ yields $\AggModels = \{\model_1, \model_2, \model_3\}$ and $\AggModels(\inn)=0$. Then, the set of naive CEs is $\CFXs^n = \{\cfx_1, \cfx_2, \cfx_3\}$, and thus augmented ensembling gives $S^n = \{\model_1, \model_2, \model_3, \cfx_1, \cfx_2, \cfx_3\}$. 
    Now, assume that $\cfx_1$ is invalid for $\model_2$ (i.e. $\model_2(\cfx_1) = 0$), $\cfx_2$ is invalid for $\model_1$, $\cfx_3$ is invalid for $\model_2$, and all three CEs are otherwise valid for all models in $\AggModels$. Then, then the set of valid CEs 
    is empty, i.e. $\CFXs^{v} = \emptyset$, and thus robust ensembling gives $S^v = \{\model_1, \model_2, \model_3 \}$.
\end{example}

This example shows that both methods host major drawbacks: augmented ensembling may produce CEs which are invalid and thus it is not robust to  MM, while robust ensembling is prone to returning 
no CEs. 
We now present a theoretical analysis to assess the extent to which augmented and robust ensembling are able to satisfy the properties given in Definitions~\ref{def:non-emptiness} to \ref{def:explanation-coherence}.

\begin{theorem}
\label{thm:s1}
    Augmented ensembling satisfies non-emptiness, model agreement, majority vote and counterfactual coherence. It satisfies non-triviality if $|\Models| \!>\! 2$. It does not satisfy counterfactual validity.
\end{theorem}

\begin{proof}
    By Def
    .~\ref{def:naive_ensembling}, it can be seen by inspection that $|\AggModels|\!\!>\!\! 0$. Thus, non-emptiness is satisfied.
    Again 
    by inspection of the same definition, 
    $\forall \model_i, \model_j \in \AggModels$, $\model_i(\inn) = \model_j(\inn)$. Thus, model agreement is satisfied. 
    We can also see that $\nexists \class_i \in \Classes \setminus \{ \AggModels(\inn) \}$ such that $| \{ \model_j \in \Models | \model_j(\inn) = \class_i \} | > | \AggModels |$. Thus, majority vote is satisfied. 
    By Defs
    .~\ref{def:naive_ensembling} and \ref{def:baselinecfxs}, it can be seen that $\forall \model_i \in \Models$ and $\forall \cfx_i \in \CFXs$, $\model_i \in S$ iff $\cfx_i \in S$. Thus, counterfactual coherence is satisfied.

    Example \ref{ex:baselines_example}
     shows that counterfactual validity is not satisfied by providing a counterexample.
    
    Finally, the partial satisfaction of non-triviality can be proven by contradiction: assume $|\Models| \!=\! n$,  $n \!>\! 2$ but $|\AggModels| \!=\! 1$. By Def
    .~\ref{def:naive_ensembling}, $\AggModels$ is the largest subset of $\Models$ containing models with the same classification outcome. However, for binary classification, this implies that the remaining $n-1$ all agree on the opposite classification, i.e. $|\Models \!\setminus\! \AggModels| \!>\! |\AggModels| $, which leads to a contradiction.
\end{proof}

\begin{theorem}
\label{thm:s2}
    Robust ensembling satisfies model agreement, majority vote and counterfactual validity. It satisfies non-triviality if $|\Models| > 2$. It does not satisfy non-emptiness or counterfactual coherence.
\end{theorem}

\begin{proof}
    The proofs for model agreement, majority vote and non-triviality are analogous to those in Theorem \ref{thm:s1} and so are omitted.

    It can be seen by inspection of Definition \ref{def:baselinecfxs} that $\forall \!\model_i \!\in\! \AggModels$ and $\forall \cfx_j \!\in \!\CFXs^v$, $\model_i(\cfx_j) \!\neq \!\model_i(\inn)$. Thus, counterfactual validity is satisfied.

    Example \ref{ex:baselines_example} shows that non-emptiness and counterfactual coherence are not satisfied by providing a counterexample.
\end{proof}

These results, summarised in Table \ref{table:props},
demonstrate that there may exist cases in which both augmented and robust ensembling fail to solve \Problem{} satisfactorily. This has strong implications on the quality of the results obtained in practice, as we will show experimentally in §\ref{sec:evaluation}. Further, these methods provide no way to take into account users' preferences over the models. 
As previously mentioned (see §\ref{sec:intro} and §\ref{sec:related}), there could be different characteristics among the models in $\Models$ in terms of meta-evaluation aspects, e.g. a model's fairness, robustness, and simplicity. Depending on the task, a model's fairness might be specified as being more important than its robustness or simplicity. In such cases, it would be desirable to have a principled way to rank models according to the preference specification, as promoted in \cite{Black_22}. 
The combination of these deficiencies motivates the need for a richer ensembling framework to solve \Problem{} while incorporating user preferences, 
given next.

\begin{table}[t]
    \begin{center}
    \resizebox{\columnwidth}{!}{
    \begin{tabular}{cccc}
    \cline{2-4}
     &
    $\AggModels \cup \CFXs^{n}$ &
    $\AggModels \cup \CFXs^{v}$ &
    $\Models^a \cup \CFXs^{a}$ \\
    \hline
    non-emptiness &
    $\checkmark$ &
     &
    $\checkmark$ \\
    non-triviality &
    $\checkmark^*$ &
    $\checkmark^*$ &
    $\checkmark^*$ \\
    model agreement &
    $\checkmark$ &
    $\checkmark$ &
    $\checkmark$ \\
    majority vote &
    $\checkmark$ &
    $\checkmark$ &
     \\
    counterfactual validity &
     &
    $\checkmark$ &
    $\checkmark$ \\
    counterfactual coherence &
    $\checkmark$ &
     &
    $\checkmark$ \\
    \hline
    \end{tabular}
    }
    \end{center}
    \protect\caption{Augmented ($\AggModels \cup \CFXs^{n}$) and robust ($\AggModels \cup \CFXs^{v}$) ensembling, as well as our argumentative approach ($\Models^a \cup \CFXs^{a}$, defined in §\ref{sec:main}), assessed against the desirable properties defined in §\ref{ssec:desirable_properties}. Satisfaction of a property is 
    shown
    by $\checkmark$, while partial satisfaction under given conditions is 
    shown
    by $\checkmark^*$.} 
    
    \label{table:props}
    \end{table}


\section{Argumentative Ensembling}
\label{sec:main}

We now present our method for ensembling models and CEs which inherently supports specifying preferences over models; then we undertake a theoretical analysis of its properties.

\subsection{Definition}

We start by formalising 
ways to incorporate the 
aforementioned preferences over models.
These preferences could be obtained by any information, e.g. meta-rules over 
models as suggested in \cite{Black_22}, but we will assume that they are extracted wrt 
properties of the models, e.g. their accuracy or (a metric representing) their simplicity.

\begin{definition}
\label{def:prop}
     Given a set  $\Models$ of models, a set $\Props$ of \emph{properties} is such that $\forall \prop \in \Props$, $\prop: \Models \rightarrow \mathbb{R}$ is a total function.
\end{definition} 

We then define a preference over the properties such that users can impose a ranking of priority over them. Here and onwards, for simplicity we use total orders, denoted $\preceq$, over any set $S$ such that, as usual, for any $s_i, s_j \in S$, $s_i \prec s_j$ iff $s_i \preceq s_j$ and $s_i \not{\succeq} s_j$. Also as usual, we say that $s_i \simeq s_j$ iff $s_i \preceq s_j$ and $s_i \succeq s_j$.

\begin{definition}
\label{def:proppref}
    Given a set  $\Props$ of properties, a \emph{property preference} $\ppreceq$ is a total order over $\Props$.
\end{definition}

Model preferences 
can be defined using property preferences.
In the following example, we define one way for 
doing so.

\begin{example}
\label{ex:CEs}
    Consider the same set of models as in Example \ref{ex:baselines_example} and a set of properties $\Props \!\!=\!\! \{ \prop_1, \prop_2 \}$ 
    where $\prop_1, \prop_2$ represent 
    model accuracy and simplicity, \resp, with a property preference $\ppreceq$ such that $\prop_1 \!\!\psucc\!\! \prop_2$ and values for the properties as follows. 

    \begin{table}[h]
    \begin{center}
    \begin{tabular}{cccccc}
    \cline{2-6}
    &
    $\model_1$ &
    $\model_2$ &
    $\model_3$ &
    $\model_4$ &
    $\model_5$ \\
    \hline
    $\prop_1$ (accuracy) &
    0.85 &
    0.87 &
    0.86 &
    0.86 &
    0.87 \\
    $\prop_2$ (simplicity) &
    0 &
    0.75 &
    1 &
    0.5 &
    0.75 \\
    \hline
    \end{tabular}
    \end{center}
    \label{table:example}
    \end{table}
    
    A simple model preference $\mpreceq$ over $\Models$ may be such that, for any $\model_i, \model_j \in \Models$, $\model_i \msucc \model_j$ iff:
        (i) $\prop_1(\model_i) > \prop_1(\model_j)$; or (ii) $\prop_1(\model_i) = \prop_1(\model_j)$ and $\prop_2(\model_i) > \prop_2(\model_j)$.
    This $\mpreceq$ is a total order over $\Models$ and results in 
    $\model_2 \msimeq \model_5 \msucc \model_3 \msucc \model_4 \msucc \model_1$.
\end{example}

Other ways to define preferences over models from preferences over properties of models can be defined, e.g. based on more sophisticated notions of dominance. We will assume some given notion of total preference over models, as follows, ignoring how it is obtained. 

\begin{definition}
\label{def:modelpref}
    Given a set  $\Models$ of models, 
    a \emph{model preference} $\mpreceq$ over $\Models$ is a total order over $\Models$.
\end{definition}


How can we incorporate these model preferences into the ensembling, while still satisfying the properties defined in §\ref{ssec:desirable_properties}? To tackle this problem, we 
use bipolar argumentation as follows.\footnote{We considered using abstract AFs \cite{Dung_95}, but we found that bipolar AFs are more suitable, given that models and CEs can be naturally seen as supporting one another.}

\begin{definition}
\label{def:BAF}
    The \emph{BAF corresponding to} 
    input $\inn$, 
    set  $\Models 
    $ of models, 
    set  $\CFXs 
    $ of CEs and 
    model preference $\mpreceq\!$ 
     is $\langle \Args, \Atts, \Supps \rangle$ 
    with:
    \begin{itemize}
        \item $\Args = \Models \cup \CFXs$;
        \item $\Atts \subseteq ( \Models \times \Models) \cup ( \Models \times \CFXs) \cup ( \CFXs \times \Models)$ where:
        \begin{itemize}
            %
            %
            \item 
            $\forall \model_i, \model_j \!\in \!\Models$, $(\model_i,\model_j) \!\in\! \Atts$ iff $\model_i(\inn) \!\!\neq \!\!\model_j(\inn)$
            , $\model_i \!\msucceq \!\model_j$;
            %
            %
            \item 
            $\forall \model_i \!\in\! \Models$ and 
            $\cfx_j \!\in\! \CFXs$ 
            where $\model_i(\cfx_j) \!= \!\model_i(\inn)$, $(\model_i, \cfx_j) \in \Atts$ iff $\model_i \msucceq \model_j$ and $(\cfx_j, \model_i) \in \Atts$ iff $\model_j \msucceq \model_i$;
            
            %
            %
        \end{itemize}
        \item $\Supps \subseteq (\Models \times \CFXs) \cup (\CFXs \times \Models)$ where for any $\model_i \in \Models$ and 
        $\cfx_j \in \CFXs$, $(\model_i,\cfx_j), (\cfx_j,\model_i) \in \Supps$ iff $i=j$.
    \end{itemize}
\end{definition}

Here, a model attacks another model if they disagree on the prediction and the latter is not strictly preferred to the former. This means that models which are outperformed with regards to preferences must be defended by more-preferred, agreeing models in order to be considered acceptable. Models and CEs are treated similarly, with the CEs inheriting the preferences from the models by which they were generated, and attacks being present between them when the model considers the CE invalid. This, along with the fact that models and their CEs support one another, ensures that the models are inherently linked to their reasoning, in the form of their CEs, and conflicts are drawn not only when two models' predictions differ, but also when their reasoning differs.

\begin{figure}[t]
    \centering
    \includegraphics[width=0.36\textwidth,trim={0.7cm 0.7cm 0.7cm 0.7cm},clip]{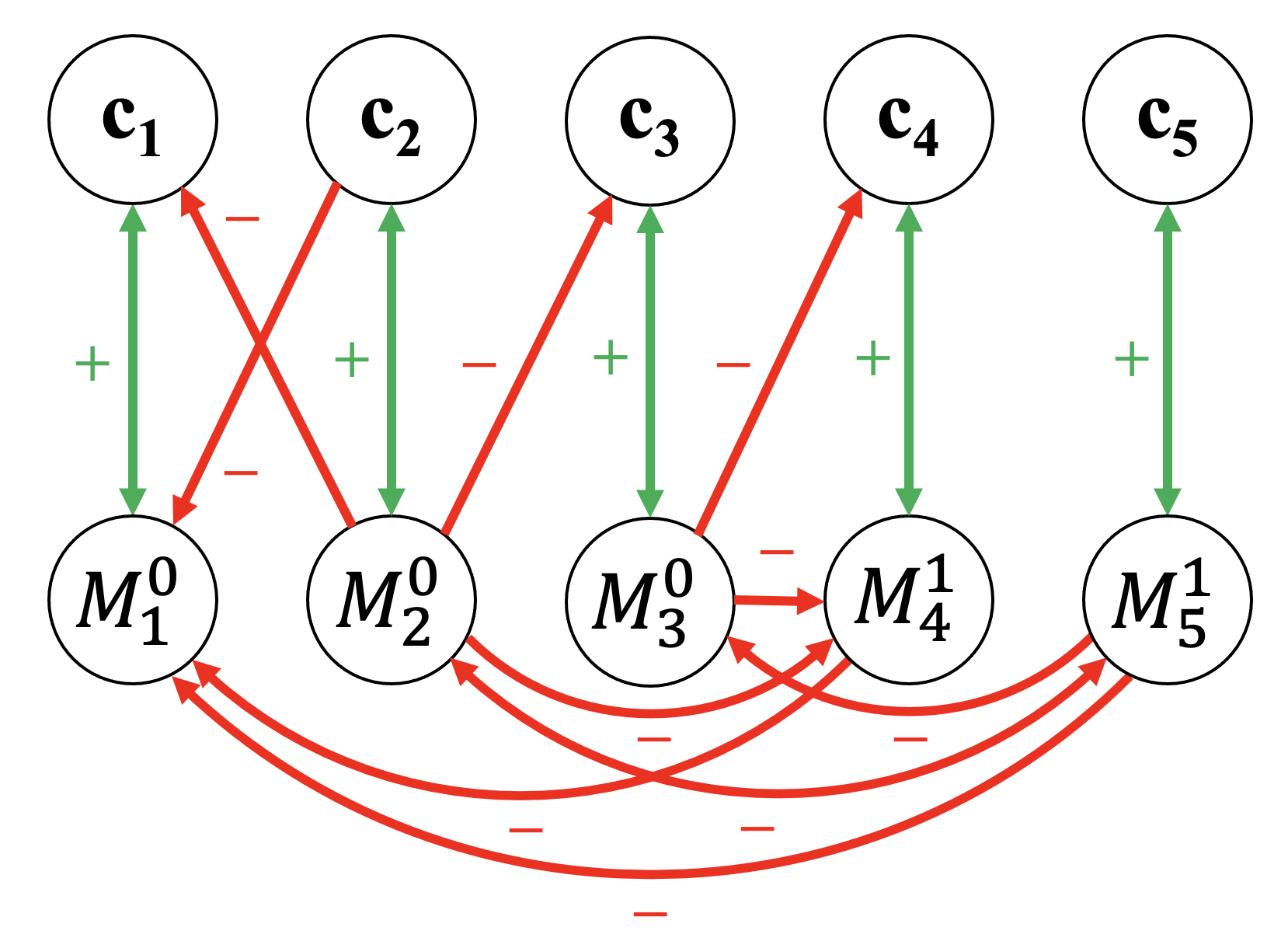}
	\caption{BAF for Example \ref{ex:arg} where: models' predictions for the input $\inn$ are given as superscripts, e.g. $\model_1(\inn) = 0$ but $\model_4(\inn) = 1$; reciprocal supports are represented by dual-headed green arrows labelled with $+$ and standard (reciprocal) attacks are represented by single-headed (dual-headed, \resp) red arrows labelled with $-$. 
    \label{fig:example_BAF}}
\end{figure}

Argumentative ensembling makes use of the set of s-preferred sets of arguments, referred to as $P_s$,  in the corresponding BAF $\langle \Args, \Atts, \Supps \rangle$, in order to resolve the 
MM problem.

\begin{definition}
\label{def:argensembling}
    Consider an input $\inn$, a set  $\Models$ of models, a set  $\CFXs$ of CEs, a set  $\Classes$ of labels and a model preference $\mpreceq$.
    Let the largest s-preferred sets for the corresponding BAF $\langle \Args, \Atts, \Supps \rangle$ be defined as:
    
    \hspace*{2cm} \(
        X_{max} = argmax_{X \in P_s} |X|. 
    \)
    \\
    Then, the solution to \Problem{} by \emph{argumentative ensembling} is defined as $S^a \in X_{max}$, where $\Models^a = S^a \cap \Models$ and $\CFXs^a = S^a \cap \CFXs$ and in the case of $| X_{max} | > 1$, we select $S^a$ from $X_{max}$ randomly.
    We also let $\Models^a(\inn) = \class_i \in \Classes$ where $\model_j(\inn) = \class_i$ $\forall \model_j \in \Models^a$.
\end{definition}

Note that, 
alternatively
, when 
$|X_{max}|>1$, 
we could choose to report all viable resulting ensembles in $X_{max}$ rather than a random one as defined above, so that more informed decisions could be made by relevant stakeholders. We leave this to future work.

The following example demonstrates how quickly the problem, when preferences are included, can become complex. This is the case even with only five models, far fewer than usual in MM.

\begin{example}
\label{ex:arg}
        The  BAF corresponding to input, models and CEs as in Example \ref{
        ex:CEs} is $\langle \Args, \Atts, \Supps \rangle$ with (see  Fig.~\ref{fig:example_BAF}):
    $\Args = $
    $\{ \model_1,$ 
    $\model_2,$ 
    $\model_3,$ 
    $\model_4,$ 
    $\model_5,$ 
    $\cfx_1,$ 
    $\cfx_2,$ 
    $\cfx_3,$ 
    $\cfx_4,$ 
    $\cfx_5 \}$; 
    $\Atts = $ 
    $\{(\model_2,\model_4), $
    $(\model_2,\model_5),$
    $(\model_2,\cfx_1),$ 
    $(\model_2,\cfx_3),$ 
    $(\model_3,\model_4\!),$
    $(\model_3,\cfx_4\!),$ 
    $(\model_4,\model_1\!),$
    $(\model_5,\model_1\!),$
    $(\model_5,\model_2\!),$
    $(\model_5,\model_3\!),$
    $(\cfx_2,\model_1\!)\!\}$; and
    $\Supps = $ 
    $\{(\model_1,\cfx_1),$ 
    $(\model_2,\cfx_2),$ 
    $(\model_3,\cfx_3),$ 
    $(\model_4,\cfx_4),$ 
    $(\model_5,\cfx_5),$  
    $(\cfx_1, \model_1),$ 
    $(\cfx_2, \model_2),$  
    $(\cfx_3, \model_3),$ 
    $(\cfx_4, \model_4),$ 
    $(\cfx_5, \model_5)\}$.
    This leads to 
    $P_s = $
    $\{ \{ \model_2,$ 
    $\cfx_2 \},$ 
    $\{ \model_4,$ 
    $\model_5,$ 
    $\cfx_4,$ 
    $\cfx_5 \} \} $, and thus 
    $S^a = $
    $\{ \model_4,$ 
    $\model_5,$ 
    $\cfx_4,$ 
    $\cfx_5 \}$ and $\Models^a(\inn) = 1$.
\end{example}

This example shows how the use of  CEs in the ensembling directly results in the prediction being reversed, relative to the other ensembling methods, violating majority vote. This is due to the fact that, while the preferences over the two sets of models are roughly similar, the validity of the CEs for the models selected by the augmented and robust ensemblings is very poor. This means that when a CE which is valid for all models is required, some compromise must be made on the model selection, as we demonstrated.

\subsection{Theoretical Analysis}

We will now undertake a theoretical analysis of 
argumentative ensembling, demonstrating some of the desirable behaviours 
thereof via properties
. First, we consider the properties introduced in §\ref{ssec:desirable_properties}.

\begin{theorem}
\label{thm:main}
    Argumentative ensembling satisfies non-emptiness, model agreement, counterfactual validity and counterfactual coherence. 
    It satisfies non-triviality if for some $\model_i \in \Models$, where $\nexists \model_j \in \Models \setminus \{ \model_i \}$ such that $\model_j \msucc \model_i$, $\exists \model_k \in \Models$ such that $\model_k(\inn) = \model_i(\inn)$, $\model_k(\cfx_i) \!\!\neq\!\! \model_k(\inn)$ and $\model_i(\cfx_k) \!\!\neq\!\! \model_i(\inn)$.
    It does not satisfy majority vote.
\end{theorem}

\begin{proof}
    Let us first prove counterfactual coherence. 
    By Def.~\ref{def:BAF},
    $\forall \model_i \in \Models$ and $\forall \cfx_j \in \CFXs$, $(\model_i, \cfx_j), (\cfx_j, \model_i) \in \Supps$. Thus, there exists an indirect attack on any $\model_i \in \Models$ iff there exists a direct attack on $\cfx_i \in \CFXs$. Likewise, there exists an indirect attack on any $\cfx_i \in \CFXs$ iff there exists a direct attack on $\model_i \in \Models$. Then, letting $\Models = \{ \model_1, \ldots, \model_m \}$ and $\CFXs = \{ \cfx_1, \ldots, \cfx_m \}$, since we know any $X \in P_s$ is maximal wrt $\Args$ by Def
    . \ref{def:argensembling}, $X$ must be such that 
    $\forall i \in \{ 1, \ldots, m \}$, 
    $\model_i \in S$ iff $\cfx_i \in S$.

    Let us prove non-emptiness by contradiction. 
    Assume that $\exists X \in P_s$ such that $X \cap \Models = \emptyset$ or $X \cap \CFXs = \emptyset$.
    We know from the above proof that, $\forall X \in P_s$, $X \cap \Models \neq \emptyset$ iff $X \cap \CFXs \neq \emptyset$.
    Then, by the definition of s-preferred extensions (see §\ref{sec:preliminaries}), it must be the case that $\forall X \in P_s$, $X = \emptyset$. 
    Based on the fact that, by Def
    . \ref{def:modelpref}, $\mpreceq$ is a total ordering and thus transitive, this is not possible as it will always be the case that $\exists \model_i \in \Models$ such that $\nexists \model_j \in \Models$ where $\model_j \msucc \model_i$. Thus, by Def
    . \ref{def:BAF}, either $\Atts(\model_i) \cup \Atts(\cfx_i) = \emptyset$ or $\forall \arg_k \in \Atts(\model_i) \cup \Atts(\cfx_i)$, $\model_i \in \Atts(\arg_k)$ or $\cfx_i \in \Atts(\arg_k)$, meaning $\{ \model_i, \cfx_i \}$ is either unattacked or is able to defend itself, therefore would be acceptable in at least one s-preferred extension, and we have the contradiction. 

    Let us 
    prove model agreement by contradiction. Assume that $\exists \model_i, \model_j \!\!\in\! \Models$ such that $\model_i(\inn) \!\neq\! \model_j(\inn)$ and $\exists X \!\!\in\! P_s$ such that $\model_i, \model_j \!\in\! X$. By Def
    .~\ref{def:BAF}, it follows that $\exists (\model_i, \model_j) \!\in\! \Atts$ or $\exists (\model_j, \model_i) \in \Atts$, which cannot be the case in an s-preferred set, which must be conflict-free (see §\ref{sec:preliminaries}), and so we have the contradiction.

    Let us prove counterfactual validity by contradiction. Assume that $\exists \model_i \!\in\! \Models$ and $\exists \cfx_j \!\in\! \CFXs$ such that $\model_i(\cfx_j) \!=\! \model_i(\inn)$ and $\exists X \!\in\! P_s$ such that $\model_i, \cfx_j \!\in\! X$. By Definition \ref{def:BAF}, it can be seen that $\exists (\model_i, \cfx_j) \!\in\! \Atts$ or $\exists (\cfx_j, \model_i) \!\in\! \Atts$, which cannot be the case in an s-preferred set, which, again, must be conflict-free, and so we have the contradiction.

    Let us prove the partial satisfaction of non-triviality by contradiction. 
    From Def.~\ref{def:argensembling} and the proof for counterfactual coherence above
    , for $|\Models^a| = 1$, it must be 
    that $\forall X \in P_s$, $| X \cap \Models| = 1$.
    However, from the above assumptions we can see that for some $\model_i \in \Models$, where $\nexists \model_j \in \Models \setminus \{ \model_i \}$ such that $\model_j \msucc \model_i$, $\exists \model_k \in \Models$ such that $\model_k(\inn) = \model_i(\inn)$, $\model_k(\cfx_i) \neq \model_k(\inn)$ and $\model_i(\cfx_k) \neq \model_i(\inn)$.
    Then, to avoid a contradiction, it must be 
    that $\exists \arg_l \in \Models \cup \CFXs$ such that $(\arg_l,\model_k) \in \Atts$ or $(\arg_l,\cfx_k) \in \Atts$, and 
    $(\model_i, \arg_l), (\cfx_i, \arg_l) \nin \Atts$.
    If $\arg_l \in \Models$, and thus $\model_l(\inn) \neq \model_k(\inn) = \model_i(\inn)$, then it must be 
    that $(\model_i, \model_l) \in \Atts$ (a contradiction, by 
    Def.~\ref{def:BAF}, since $\model_i \msucceq \model_l$).
    If $\arg_l \in \CFXs$ and $\model_l(\inn) \neq \model_k(\inn) = \model_i(\inn)$, then $(\model_i, \model_l) \in \Atts$ (a contradiction by the same reasoning).
    Finally, if $\arg_l \in \CFXs$ and $\model_l(\inn) = \model_k(\inn) = \model_i(\inn)$, then 
    either: if $\model_i(\cfx_l) = \model_i(\inn)$ or $\model_l(\cfx_i) = \model_l(\inn)$, then $(\model_i, \cfx_l) \in \Atts$ (a contradiction); or
    , otherwise, $\model_l \in X'$ (which can be checked by repeating the steps 
    for $\model_k$, for $\model_l$ instead).
    Thus
    $\exists X' \in P_s$ where $|X' \cap \Models| > 1$, and we have the contradiction in all cases.


    Finally, Example \ref{ex:arg} provides a counterexample which shows that majority vote is not satisfied.
\end{proof}


These results contrast with those for augmented and robust ensembling, as shown in Table \ref{table:props}. 
Argumentative ensembling avoids the pitfalls of augmented and robust ensembling by satisfying non-emptiness, counterfactual validity and counterfactual coherence.  
In order to achieve this behaviour, majority vote is sacrificed as a guarantee. In §\ref{sec:evaluation} we will assess the impact of not guaranteeing majority vote on argumentative ensembling's 
accuracy, along with other metrics. 

However, first, 
we consider the relationship between the different ensembling methods.
For the remainder of the section, we assume as given an input $\inn$, a set $\Models$ of models, a set $\CFXs$ of CEs, a model preference $\mpreceq$ and 
a corresponding BAF $\langle \Args, \Atts, \Supps \rangle$ with $P_s$ the set of all its s-preferred sets.



\begin{theorem}
\label{thm:nairobarg}
    If $\forall \model_i, \model_j \in \Models$, $\model_i \msimeq \model_j$, and $\forall \model_k \in \Models$ and $\cfx_l \in \CFXs$, where $\model_k(\inn) = \model_l(\inn)$, $\model_k(\cfx_l) \neq \model_k(\cfx_l)$, then augmented, robust and argumentative ensembling are equivalent.
\end{theorem}

\begin{proof}
    
    If $\forall \model_i, \model_j \in \Models$, $\model_i \msimeq \model_j$, then it can be seen from Definition \ref{def:BAF} that $(\model_i,\model_j), (\model_j,\model_i) \in \Atts$ iff $\model_i(\inn) \neq \model_j(\inn)$. 
    Note that, by Definition \ref{def:baselinecfxs}, augmented and robust ensembling are equivalent since $\forall \model_k \in \Models$ and $\cfx_l \in \CFXs$, where $\model_k(\inn) = \model_l(\inn)$, $\model_k(\cfx_l) \neq \model_k(\cfx_l)$.
    Also by the assumptions, it can be seen from Definition \ref{def:BAF} that $\forall \model_k \in \Models$ and $\forall \cfx_l \in \CFXs$, $(\model_k,\cfx_l), (\cfx_l,\model_k) \in \Atts$ iff $\model_k(\cfx_l) = \model_k(\inn)$ and $\model_k(\inn) \neq \model_l(\inn)$ due to the assumptions in the theorem, meaning any attack is reciprocated and all arguments defend themselves. 
    Then, $\forall \model_m, \model_n \in \Models$ such that $\model_m(\inn) = \model_n(\inn)$, $(\model_m,\model_n) \nin \Atts$.
    Thus, $P_s = \{ \{ \model_o \in \Models, \cfx_o \in \CFXs | \model_o(\inn) = 0 \} , \{ \model_p \in \Models, \cfx_p \in \CFXs | \model_p(\inn) = 1 \}$. 
    By Definitions \ref{def:naive_ensembling}, \ref{def:baselinecfxs} and \ref{def:argensembling}, all forms of ensembling select from the same two sets of models and CEs in the same manner and are thus equivalent.
\end{proof}

We also provide a number of theoretical 
results concerning the behaviour of 
argumentative ensembling, first relating to the preferences. 
The first 
result demonstrates how a completely dominant model wrt the preferences will be present in all s-preferred sets.
    
\begin{proposition}
\label{prop:dominant_model}
    If $\exists \model_i \in \Models$ such that $\forall \model_j \in \Models$, $\model_i \msucc \model_j$, then $\forall X \in P_s$, $\model_i \in X$.
\end{proposition}

\begin{proof}
    If $\forall \model_j \in \Models$, $\model_i \msucc \model_j$ then, by Definition \ref{def:BAF}, $\Atts(\model_i) = \Atts(\cfx_i) =\emptyset$ and $\forall \model_j \in \Models$ where $\model_j(\inn) \neq \model_i(\inn)$, $\model_i \in \Atts(\model_j)$. Similarly, $\forall \cfx_k \in \CFXs$  
    where $\model_i(\cfx_k)=\model_i(\inn)$, $\model_i \in \Atts(\cfx_k)$, and so $\model_i$ indirectly attacks $\model_k$. Then, $\nexists X \in P_s$ such that $\model_j \in X$ or $\model_k \in X$ and thus, $\forall X \in P_s$, $\model_i \in X$.
\end{proof}

We also show how, for any two s-preferred sets, there exists some trade-off between their models wrt the preferences.

\begin{lemma}
    Given two s-preferred sets $X, X' \in P_s$, $\nexists \model_i \in (X \cap \Models) \setminus X'$ such that $\model_i \msucc \model_j$ $\forall \model_j \in X'$. 
\end{lemma}


Meanwhile, in the non-strict case, a model which is not 
outperformed by any other model wrt the preferences will be present in at least one s-preferred set.

\begin{proposition}
\label{prop:the_best_model_is_in_m_preferred_set}
    If $\exists \model_i \in \Models$ such that $\forall \model_j \in \Models$, $\model_i \msucceq \model_j$, then $\exists X \in P_s$ such that $\model_i \in X$. 
\end{proposition}

\begin{proof}
    If $\forall \model_j \in \Models$, $\model_i \msucceq \model_j$ then, by Definition \ref{def:BAF}, $\nexists \arg_k \in \Atts(\model_i) \cup \Atts(\cfx_i)$ such that $\model_i \nin \Atts(\arg_k)$ and $\cfx_i \nin \Atts(\arg_k)$. Thus, it must be the case that $\exists X \in P_s$ such that $\model_i \in X$.
\end{proof}

We now show that if a model is outperformed wrt the preferences by all other models, then the outperformed model cannot exist in an s-preferred set unless it is defended by a more preferred model.

\begin{proposition}
    For any $\model_i \in \Models$, if $\model_j \msucc \model_i$ $\forall \model_j \in \Models \setminus \{ \model_i \}$ and $\exists X \in P_s$ such that $\model_i \in X$, then, $\forall \model_k \in \Models$ where $\model_k(\inn) \neq \model_i(\inn)$,  
    $\exists \model_l \in X$ such that $\model_l \msucceq \model_k$.
\end{proposition}

\begin{proof}
    By Definition \ref{def:BAF}, $\model_k \!\in \!\Atts(\model_i)$ and $\{ \model_m \!\in\! \Models | \model_i \!\in\! \Atts (\model_m) \!\vee\! \cfx_i \!\in\! \Atts (\model_m) \} = \emptyset$. 
    Then, given that $\exists X \!\in\! P_s$ such that $\model_i \!\in\! X$, we know that, $\forall \model_k \!\in\! \Atts(\model_i)$ $\exists \model_l \!\in\! X$ where $\model_l(\inn) \!=\! \model_i(\inn)$ and $\model_l \!\in\! \Atts(\model_k)$. Thus, by Definition \ref{def:BAF}, $\model_l \msucceq \model_k \msucc \model_i$.
\end{proof}

We 
also consider the behaviour of 
argumentative ensembling wrt the selected CEs, demonstrating 
that those from disagreeing models are guaranteed not to be included in any s-preferred set.

\begin{proposition}
    Any s-preferred set $X \in P_s$ is such that $\nexists \cfx_i, \cfx_j \in X \cap \CFXs$ where $\model_i(\inn) \neq \model_j(\inn)$.
\end{proposition}
\begin{proof}
    Let us prove by contradiction, assuming that $\exists \cfx_i, \cfx_j \in X \cap \CFXs$.
    Counterfactual coherence (Theorem \ref{thm:main}) requires that $\model_i, \model_j \in X$. 
    However, by Definition \ref{def:BAF}, $\model_i(\inn) \neq \model_j(\inn)$ requires that $\model_i \in \Atts(\model_j)$ or  $\model_j \in \Atts(\model_i)$, and so we have the contradiction.
\end{proof}

Finally, we show that s-preferred sets 
of corresponding BAFs in our setting satisfy all the forms of admissibility for BAFs in \cite{Cayrol_05}
.

\begin{proposition}
\label{prop:admissibility}
    Any s-preferred set $X \in P_s$ is d-admissible, s-admissible and c-admissible. 
\end{proposition}

\begin{proof}
    Trivially, any s-preferred set is s-admissible and thus also d-admissible (see §\ref{sec:preliminaries}).
    Then, it can be seen from Definition \ref{def:BAF} that $\Supps \subseteq (\Models \times \CFXs) \cup (\CFXs \times \Models)$, $i=j$ $\forall (\model_i, \cfx_j) \in \Supps \cap (\Models \times \CFXs)$ and $k=l$ $\forall (\cfx_k, \model_l) \in \Supps \cap (\CFXs \times \Models)$.  
    By counterfactual coherence (Theorem \ref{thm:main}), $\forall \model_i \in \Models$ and $\forall \cfx_i \in \CFXs$, $\model_i \in X$ iff $\cfx_i \in X$.
    Then, since $P_s$ contains the sets of $\Args$ which are maximal wrt set-inclusion, any $X' \in P_s$ must be closed for $\Supps$ and thus c-admissible.
\end{proof}

\section{Empirical Evaluation}
\label{sec:evaluation}





We now examine the effectiveness of our approach using three real-world datasets. Specifically, we empirically 
evaluate the extent to which each of the ensembling 
methods introduced in §\ref{ssec:naive} and §\ref{sec:main} satisfy the desirable properties defined in §\ref{ssec:desirable_properties}. 
We also instantiate three variations of 
argumentative ensembling by including two different types of model properties 
$\Props$ and demonstrate the usefulness of incorporating model preferences into ensembling methods.


\subsection{Experiment Setup} We apply all ensembling methods on three datasets in the legal and financial contexts: loan approval (heloc) \cite{heloc}, recidivism prediction (compas) \cite{compas}, and credit risk (credit) \cite{Dua2019}. Due to 
neural networks' sensitivity to randomness at training time, 
they suffer severely from MM and are frequently targeted when investigating this research topic (as discussed in §\ref{sec:related}). Therefore, even though our method is model-agnostic, we focus on neural networks for the experiments.

For each dataset, we train-test 150 classifiers with five different hidden layer sizes using 80\% of the dataset (this 80\% is train-test split for training each model; see Appendix \ref{app:exp-dataset-training} for dataset and training details). The 150 neural networks are trained using different random seeds for parameter initialisation and 
different train-test splits (within the train-test 80\% of the dataset), forming a pool of possible models under MM from which we sample multiple sets  $\Models$ of models to which we apply our ensembling methods. 
We use the remaining 20\% of each dataset as test inputs for the ensembling methods (limited to 500 inputs test set if larger than this). 

At each run, we randomly sample, from the model pool, sets $\Models$ with 10, 20 or 30 models
, then we 
feed each input to the models to receive their predicted labels and generate one CE from each model using the nearest neighbour CEs approach of \cite{NiceNNCE}, and finally apply the ensembling methods. For each size (10, 20, 30), we 
perform five different 
choices of 
$\Models$, and record the 
mean and standard deviation of the results 
(over the five choices of model sets for each size). 

As concerns model preferences, we 
focus on 
accuracy of the trained classifiers over the (20\%) test inputs and model structure simplicity. 
For the latter, we assign the models, from the most complex to the simplest (depending on the number of neurons in the hidden layers, see Appendix \ref{app:exp-dataset-training}), scores of \{0, 0.25, 0.5, 0.75, 1\} such that higher values imply simpler models. Note that 
multiple models in $\Models$ may have the same simplicity scores 
as
we adopt only five different model structures to obtain 
150 neural networks for each dataset. Models in $\Models$ may have any (near-optimal) test accuracy: in the experiments each such
model 
has a different accuracy.

\begin{table*}[ht!]
    \centering
    \resizebox{2\columnwidth}{!}{
    \begin{tabular}{c|ccccc|ccccc|ccccc}
     
    & \textbf{acc.} & \textbf{simp.} & \textbf{size M}/\textbf{C} &    \textbf{$\cfx$ val. (fail)} & \textbf{mv}
    & \textbf{acc.} & \textbf{simp.} & \textbf{size M}/\textbf{C} & \textbf{$\cfx$ val. (fail)} & \textbf{mv}
    & \textbf{acc.} & \textbf{simp.} & \textbf{size M}/\textbf{C} & \textbf{$\cfx$ val. (fail)} & \textbf{mv}\\

    \hline
    &
    \multicolumn{5}{c|}{heloc} &
    \multicolumn{5}{c|}{compas} &
    \multicolumn{5}{c}{credit} \\
    \hline
    \textbf{$|\Models|=10$}  
    & \multicolumn{5}{l|}{.709$\pm$.003} 
    & \multicolumn{5}{l|}{.856$\pm$.001}
    & \multicolumn{5}{l}{.664$\pm$.008}\\

    $S^n$ 
    & .709 & .495 & .943/.943 & .657 (.00) & 1.00
    & .858 & .464 & .980/.980 & .572 (.00) & 1.00
    & .697 & .588 & .817/.817 & .757 (.00) & 1.00\\
    $S^v$ 
    & .709 & .495 & .943/.309 & 1.00 (.34) & 1.00
    & .858 & .464 & .980/.174 & 1.00 (.51) & 1.00
    & .697 & .588 & .817/.457 & 1.00 (.44) & 1.00\\
    
    \textbf{$S^a$}
    & .712 & .504 & .499/.499 & 1.00 (.00) & .983
    & .859 & .463 & .369/.369 & 1.00 (.00) & .994
    & .694 & .600 & .580/.580 & 1.00 (.00) & .953\\
    \textbf{$S^a$-A} 
    & .726 & .485 & .357/.357 & 1.00 (.00) & .943
    & .864 & .430 & .295/.295 & 1.00 (.00) & .988
    & .710 & .593 & .486/.486 & 1.00 (.00) & .825\\
    \textbf{$S^a$-S}
    & .710 & .608 & .462/.462 & 1.00 (.00) & .967
    & .860 & .657 & .306/.306 & 1.00 (.00) & .987
    & .689 & .626 & .565/.565 & 1.00 (.00) & .925\\
    \textbf{$S^a$-AS} 
    & .712 & .528 & .493/.493 & 1.00 (.00) & .980
    & .860 & .501 & .360/.360 & 1.00 (.00) & .994
    & .696 & .607 & .578/.578 & 1.00 (.00) & .946\\

    \hline
    \textbf{$|\Models|=20$}      
    & \multicolumn{5}{l|}{.710$\pm$.003} 
    & \multicolumn{5}{l|}{.855$\pm$.001}
    & \multicolumn{5}{l}{.663$\pm$.004}\\
    $S^n$ 
    & .717 & .488 & .940/.940 & .626 (.00) & 1.00
    & .859 & .538 & .978/.978 & .544 (.00) & 1.00
    & .708 & .571 & .810/.810 & .734 (.00) & 1.00\\
    $S^v$ 
    & .717 & .388 & .940/.230 & 1.00 (.37) & 1.00
    & .859 & .538 & .978/.111 & 1.00 (.60) & 1.00
    & .708 & .571 & .810/.351 & 1.00 (.62) & 1.00\\
    
    \textbf{$S^a$}
    & .716 & .466 & .460/.460 & 1.00 (.00) & .984
    & .859 & .514 & .331/.331 & 1.00 (.00) & .992
    & .691 & .580 & .557/.557 & 1.00 (.00) & .961\\
    \textbf{$S^a$-A} 
    & .728 & .432 & .361/.361 & 1.00 (.00) & .950
    & .866 & .541 & .235/.235 & 1.00 (.00) & .982
    & .709 & .586 & .481/.481 & 1.00 (.00) & .862\\
    \textbf{$S^a$-S} 
    & .711 & .551 & .420/.420 & 1.00 (.00) & .966
    & .857 & .609 & .304/.304 & 1.00 (.00) & .987
    & .684 & .590 & .549/.549 & 1.00 (.00) & .947\\
    \textbf{$S^a$-AS} 
    & .715 & .473 & .459/.459 & 1.00 (.00) & .984
    & .859 & .555 & .324/.324 & 1.00 (.00) & .990
    & .693 & .581 & .556/.556 & 1.00 (.00) & .959\\

    \hline
    \textbf{$|\Models|=30$}  
    & \multicolumn{5}{l|}{.710$\pm$.003} 
    & \multicolumn{5}{l|}{.855$\pm$.001}
    & \multicolumn{5}{l}{.663$\pm$.004}\\
    $S^n$ 
    & .718 & .512 & .940/.940 & .620 (.00) & 1.00
    & .859 & .527 & .976/.976 & .530 (.00) & 1.00
    & .710 & .540 & .807/.807 & .727 (.00) & 1.00\\
    $S^v$ 
    & .718 & .512 & .940/.205 & 1.00 (.41) & 1.00
    & .859 & .527 & .976/.087 & 1.00 (.56) & 1.00
    & .710 & .540 & .807/.311 & 1.00 (.70) & 1.00\\
   
    \textbf{$S^a$}
    & .716 & .499 & .456/.456 & 1.00 (.00) & .982
    & .862 & .519 & .308/.308 & 1.00 (.00) & .990
    & .683 & .549 & .551/.551 & 1.00 (.00) & .943\\
    \textbf{$S^a$-A} 
    & .729 & .406 & .353/.353 & 1.00 (.00) & .946
    & .865 & .532 & .225/.225 & 1.00 (.00) & .983
    & .711 & .552 & .441/.441 & 1.00 (.00) & .850\\
    \textbf{$S^a$-S} 
    & .712 & .518 & .445/.445 & 1.00 (.00) & .978
    & .861 & .567 & .294/.294 & 1.00 (.00) & .988
    & .684 & .555 & .546/.546 & 1.00 (.00) & .940\\
    \textbf{$S^a$-AS}
    & .716 & .500 & .456/.456 & 1.00 (.00) & .982
    & .862 & .543 & .303/.303 & 1.00 (.00) & .988
    & .683 & .549 & .551/.551 & 1.00 (.00) & .944\\

    \end{tabular}
    }
    \caption{Quantitative evaluations of ensembling methods on three datasets, heloc, compas, and credit. $|\Models|=\{10, 20, 30\}$ stands for results for different model set sizes, the acc. entries in the rows starting with $|\Models|$ are the average single model accuracies.}
    \label{tab:results-accuracy-cfx-validity}
\end{table*}

\textit{Evaluation metrics.} Each ensembling method is evaluated against the following metrics: prediction accuracy over the test set (\emph{acc.}), average model simplicity in the ensemble (\emph{simp.}), average size of models and CEs in the ensemble, measured as percentages of $|\Models|$ (\emph{size M/C}), average validities of ensembled CEs over the ensembled models (\emph{$\cfx$ val.}). Also, we report the percentage of test inputs for which a method fails to produce 
CEs (fail). The results are averaged over all the test inputs except for the failure cases. We also report the average test set accuracies of the models in $\Models$. Note that the model agreement property (Property~\ref{def:model_agreement}) is omitted as it is satisfied by every compared method
. To understand how the violation of the majority vote property affects our method, we measure the proportion of test inputs for which the predicted label using our method is the same as that of naive ensembling (\emph{mv}).

\textit{Ensembling Methods.} We use augmented 
and robust ensembling as 
baselines. 
For 
argumentative ensembling, we use four variations with different preferences
: $S^a$ ($\Props\!=\!\emptyset$), $S^a$-A ($\Props\!=\!\{accuracy\}$), $S^a$-S ($\Props\!=\!\{simplicity\}$) and $S^a$-AS ($\Props\!=\!\{accuracy$, $simplicity\}$, 
where $accuracy \!\psimeq\! simplicity$). 
In our implementation of argumentative ensembling, when $|X_{max}|\!>\!1$ (
Definition \ref{def:argensembling}), we return $S^a$, which has the same prediction label as naive ensembling. We 
give percentages of test inputs with $|X_{max}\!\!>\!\!1|$ in Table 3 in Appendix \ref{app:exp-multiple-solutions}.




\subsection{Results and Analyses}

We report the results for all experiments in Table \ref{tab:results-accuracy-cfx-validity} (the standard deviations are presented in Tables 4 to 6 in Appendix \ref{app:exp-std}).

\textit{Usefulness of preferences.} 
With test accuracy specified as model preference, $S^a$-A shows the best accuracy in all experiments. This validates Proposition \ref{prop:dominant_model}, because, assuming that the accuracy for every model in $\Models$ is different, for $S^a$-A, there exists a model in 
$\Models$ that is the most preferred 
and is included in the ensemble. Similarly, $S^a$-S shows the best simp. scores in all experiments. However, since simplicity scores are not unique for each model, usually a single most preferred model does not exist, therefore an optimal simp. evaluation is not guaranteed. 
When specifying both properties as model preferences ($S^a$-AS), at least one of the two metrics is improved compared with $S^a$.



\textit{Desirable properties of ensembling methods.} 
For up to 70\% of test inputs, robust ensembling does not find any 
CEs ($S^v \cap \CFXs=\emptyset$), 
confirming its violation of non-emptiness.
As $|\Models|$ increases, $S^v$ would require finding CEs which are valid for more models, and the number of CEs found would drop as shown by the results for the size C evaluations. In contrast, the remaining methods, including 
argumentative ensembling, always find non-empty ensembles, the sizes of which are also not affected by the model set sizes.


$S^n$ demonstrates low 
c Val. scores, showing that, on average, a CE from a model 
in the ensemble is only valid for 53.0\% to 75.7\% of other agreeing models. Therefore, the violation of counterfactual validity has a significant impact on results in practice. $S^v$ produces valid CEs over models in the ensemble, however, as previously mentioned, they do not always exist.

For $S^a$, we 
note
the same number of models and CEs in the solution set and 100\% counterfactual validity, confirming the behaviour predicted by 
Theorem \ref{thm:main}.
Argumentative ensembling shows model and CE ensemble sizes of 22.5\% (when $|\Models|=30$) to 58.0\% of $|\Models|$, meaning that 
it is non-trivially more selective than $S^n$ and $S^v$, only accepting the largest set of models with similar reasoning local to the test input (validated by agreement on CEs as required by counterfactual coherence). This results in comparable test accuracies as $S^n$ and $S^v$ with guaranteed CE validity. In fact, 
mostly, the $S^a$-S option has the lowest agreement rate with majority vote prediction 
(mv), 
but
it is more accurate than the baselines 
using naive ensembling. When no preference is specified, 
argumentative ensembling has higher accuracy than majority vote for 
heloc 
when $|\Models|=10$ and for 
compas 
when $|\Models|=10,30$. 
Thus,
we do not necessarily lose accuracy in satisfying  
properties besides majority vote.

\section{Conclusions and Future Work}
\label{sec:conclusion}

We
have presented a formal study of the problem of providing recourse under MM.
We defined several 
properties which are desirable in methods for solving this problem, highlighting deficiencies in 
extending conservatively the standard naive ensembling used for MM without recourse.
We have then introduced argumentative ensembling, a novel method for providing recourse under MM, which leverages computational argumentation to incorporate robustness guarantees and user preferences over models. 
We show, by means of a theoretical analysis, that argumentative ensembling hosts advantages over 
other methods, notably in non-emptiness of solutions and validity of CEs, notwithstanding its ability to handle user preferences.
This is, however, achieved by sacrificing the satisfaction of the property of majority vote, and so we conducted an empirical analysis with three real-world datasets to examine the effects of this 
sacrifice.
Our results demonstrate that 
argumentative ensembling always finds valid CEs without compromising prediction accuracy, and shows the usefulness of specifying preferences over models.

This paper opens up several potentially fruitful directions for future work.
First, it would be interesting to examine the extent to which considering attacks to or from \emph{sets} of arguments, rather than single arguments, as in \cite{Nielsen_06,Flouris_19,Dvorak_22,Dimopoulos_23}, 
may help in MM given the conflicts between agreeing sets of models.
Further, extended AFs~\cite{Modgil_09} and value-based AFs \cite{value-based} may provide useful alternative ways to account for preferences.
We would also like to exploit the explanatory potential of argumentation to 
support explainable ensembling, e.g. using sub-graphs as in~\cite{Fan_14,Zeng_19}.
Moreover, in order to support experiments with a high number of models (beyond the 30 we considered),  large-scale argumentation solvers would be highly desirable.
Finally, it would be interesting to assess the effect which MM has on users' evaluations of CEs.

\section*{Acknowledgement}
{Jiang, Rago and Toni were partially funded by J.P. Morgan and by the Royal Academy of Engineering under the Research Chairs and Senior Research Fellowships scheme. 
Leofante is supported by an Imperial College Research Fellowship grant. 
Rago and Toni were partially funded by the European Research Council (ERC) under the European Union’s Horizon 2020 research and innovation programme (grant agreement No. 101020934). 
Any views or opinions expressed herein are solely those of the authors listed.}

\bibliographystyle{named}
\bibliography{bib}

\clearpage
\newpage
\appendix

\section*{Appendices}
\appendix

\section{Dataset and Training Details}
\label{app:exp-dataset-training}

The neural networks we train for all datasets are with two hidden layers of sizes \{(15, 15), (20, 15), (20, 20), (30, 20), (30, 25)\}. The five different sizes are for obtaining five levels of model simplicity which serves as a toy property evaluation to demonstrate the usefulness of supporting model preferences in our framework. We observe no significant changes in test accuracies across the different model sizes, as the accuracy variations fall within the range of standard deviation. We use the \texttt{sklearn} implementation of neural networks for the experiment (https://scikit-learn.org/stable/modules/generated/\\sklearn.neural\_network.MLPClassifier.html) and train with a batch size of 64. All experiments are executed on a standard Windows machine with an Intel Core i7-12700H CPU with 40GB memory.

The heloc dataset has 9871 data points with 21 continuous features. The 5-fold cross-validation accuracies for each size are, respectively, .730$\pm$.015, .738$\pm$.009, .741$\pm$.008, .734$\pm$.007, .736$\pm$.009.

The compas dataset has 6172 data points with 3 discrete and 4 continuous features. The 5-fold cross-validation accuracies for each size are .844$\pm$.013, .841$\pm$.013, .845$\pm$.011, .840$\pm$.012, .845$\pm$.012.

There are 1000 data points and 20 categorical features in the credit dataset. The 5-fold cross-validation accuracies for each size are .709$\pm$.028, .699$\pm$.007, .705$\pm$.018, .697$\pm$.011, .696$\pm$.021.

\section{When Multiple Solutions Exist}
\label{app:exp-multiple-solutions}

\begin{table*}[ht!]
    \centering

    \begin{tabular}{c|cc|cc|cc}
     
    & \textbf{multiple.} & \textbf{same.}
    & \textbf{multiple.} & \textbf{same.}
    & \textbf{multiple.} & \textbf{same.} \\

    \hline
    &
    \multicolumn{2}{c|}{heloc} &
    \multicolumn{2}{c|}{compas} &
    \multicolumn{2}{c}{credit} \\
    \hline
    \textbf{$|\Models|=10$}  
    & \multicolumn{2}{l|}{} 
    & \multicolumn{2}{l|}{}
    & \multicolumn{2}{l}{}\\

    $S^n$ 
    & .022$\pm$.003 & .000$\pm$.000  
    & .005$\pm$.004 & .000$\pm$.000  
    & .074$\pm$.016 & .000$\pm$.000  \\
    $S^v$ 
    & .022$\pm$.003 & .000$\pm$.000  
    & .005$\pm$.004 & .000$\pm$.000  
    & .074$\pm$.016 & .000$\pm$.000  \\
    
    \textbf{$S^a$}
    & .164$\pm$.002 & .145$\pm$.004
    & .256$\pm$.008 & .248$\pm$.008
    & .100$\pm$.015 & .046$\pm$.013\\
    \textbf{$S^a$-A} 
    & .000$\pm$.000 & .000$\pm$.015
    & .084$\pm$.010 & .083$\pm$.103
    & .015$\pm$.029 & .010$\pm$.019\\
    \textbf{$S^a$-S}
    & .097$\pm$.005 & .088$\pm$.026
    & .212$\pm$.012 & .207$\pm$.012
    & .072$\pm$.025 & .036$\pm$.008\\
    \textbf{$S^a$-AS} 
    & .137$\pm$.002 & .122$\pm$.005
    & .211$\pm$.005 & .206$\pm$.005
    & .089$\pm$.019 & .040$\pm$.010\\

    \hline
    \textbf{$|\Models|=20$}  
    & \multicolumn{2}{l|}{} 
    & \multicolumn{2}{l|}{}
    & \multicolumn{2}{l}{}\\

    $S^n$ 
    & .009$\pm$.002 & .000$\pm$.000 
    & .002$\pm$.002 & .000$\pm$.000  
    & .033$\pm$.006 & .000$\pm$.000  \\
    $S^v$ 
    & .009$\pm$.002 & .000$\pm$.000  
    & .002$\pm$.002 & .000$\pm$.000  
    & .033$\pm$.006 & .000$\pm$.000  \\
    
    \textbf{$S^a$}
    & .105$\pm$.006 & .095$\pm$.006
    & .178$\pm$.052 & .174$\pm$.048
    & .057$\pm$.006 & .023$\pm$.006\\
    \textbf{$S^a$-A} 
    & .018$\pm$.035 & .016$\pm$.033
    & .027$\pm$.054 & .027$\pm$.053
    & .014$\pm$.020 & .007$\pm$.012\\
    \textbf{$S^a$-S}
    & .078$\pm$.039 & .072$\pm$.037
    & .136$\pm$.028 & .132$\pm$.028
    & .047$\pm$.009 & .024$\pm$.011\\
    \textbf{$S^a$-AS} 
    & .102$\pm$.009 & .093$\pm$.009
    & .146$\pm$.024 & .142$\pm$.022
    & .059$\pm$.006 & .024$\pm$.007\\
    \hline
    \textbf{$|\Models|=30$}  
    & \multicolumn{2}{l|}{} 
    & \multicolumn{2}{l|}{}
    & \multicolumn{2}{l}{}\\

    $S^n$ 
    & .009$\pm$.002 & .000$\pm$.000  
    & .002$\pm$.001 & .000$\pm$.000  
    & .022$\pm$.008 & .000$\pm$.000  \\
    $S^v$ 
    & .009$\pm$.002 & .000$\pm$.000  
    & .002$\pm$.001 & .000$\pm$.000  
    & .022$\pm$.008 & .000$\pm$.000  \\
    
    \textbf{$S^a$}
    & .078$\pm$.013 & .072$\pm$.011
    & .133$\pm$.032 & .130$\pm$.032
    & .038$\pm$.013 & .019$\pm$.009\\
    \textbf{$S^a$-A} 
    & .000$\pm$.000 & .000$\pm$.000
    & .035$\pm$.050 & .035$\pm$.050
    & .010$\pm$.020 & .005$\pm$.010\\
    \textbf{$S^a$-S}
    & .066$\pm$.019 & .062$\pm$.016
    & .093$\pm$.011 & .091$\pm$.011
    & .030$\pm$.009 & .012$\pm$.005\\
    \textbf{$S^a$-AS} 
    & .076$\pm$.012 & .070$\pm$.012
    & .102$\pm$.022 & .098$\pm$.021
    & .038$\pm$.012 & .019$\pm$.007\\
    \end{tabular}
    
    \caption{The proportion of inputs for which multiple solutions and multiple solutions with the same prediction label exist.}
    \label{tab:results-multiple-solutions}
\end{table*}

We report for all ensembling methods the proportion of test inputs for which multiple solutions exist in Table~\ref{tab:results-multiple-solutions} (\emph{multiple.}). In the case of naive ensembling, whenever multiple solutions are found, they have different prediction labels. For our argumentative ensembling method, it could be the case that multiple solutions give the same prediction result in which case the desirable properties satisfaction is not compromised, we thus also include the proportion of inputs where this happens (reported as \emph{same.}). 

Though the number of cases when $|X_{max}|>1$ can be relatively high for our method, they mostly provide the same prediction results for the input, in which case the only difference for the end user (individual data subject) is the recourse they could receive between the multiple solutions. On the other hand, compared with naive ensembling methods, it is less frequent for our method to find equally sized ensembles which disagree on the prediction for the input, observed by $(multiple. - same.)$. As mentioned in §\ref{sec:main}, in practice, one could use all viable resulting ensembles to make more informed decisions about the individual. Therefore, when multiple answers exist, argumentative ensembling provides extra flexibility for its users.

\section{Full Evaluation Results}
\label{app:exp-std}

We report the full results of Table~\ref{tab:results-accuracy-cfx-validity} including standard deviations in Tables~\ref{tab:results-std1} to \ref{tab:results-std3}.

\begin{table*}[h]
    \centering
    \begin{tabular}{c|ccccc}
     
    & \textbf{acc.} & \textbf{simp.} & \textbf{size M}/\textbf{C} &    \textbf{$\cfx$ val. (fail)} & \textbf{mv.}\\
    \hline
    &
    \multicolumn{5}{c}{heloc} \\
    \hline
    \textbf{$|\Models|=10$}  
    & \multicolumn{5}{l}{.709$\pm$.003} \\
    $S^n$ 
    & .709$\pm$.005 & .495$\pm$.122 & .943$\pm$.003/.943$\pm$.003 & .657$\pm$.020 (.00$\pm$.000) & 1.00$\pm$.000\\
    $S^v$ 
    & .709$\pm$.005 & .495$\pm$.122 & .943$\pm$.003/.309$\pm$.028 & 1.00$\pm$.000 (.34$\pm$.059) & 1.00$\pm$.000\\
    \textbf{$S^a$}
    & .712$\pm$.005 & .504$\pm$.148 & .499$\pm$.016/.499$\pm$.016 & 1.00$\pm$.000 (.00$\pm$.000) & .983$\pm$.004\\
    \textbf{$S^a$-A} 
    & .726$\pm$.006 & .485$\pm$.229 & .357$\pm$.048/.357$\pm$.048 & 1.00$\pm$.000 (.00$\pm$.000) & .943$\pm$.015\\
    \textbf{$S^a$-S}
    & .710$\pm$.007 & .608$\pm$.085 & .462$\pm$.035/.462$\pm$.035 & 1.00$\pm$.000 (.00$\pm$.000) & .967$\pm$.026\\
    \textbf{$S^a$-AS} 
    & .712$\pm$.005 & .528$\pm$.135 & .493$\pm$.017/.493$\pm$.017 & 1.00$\pm$.000 (.00$\pm$.000) & .980$\pm$.005\\

    \hline
    \textbf{$|\Models|=20$}      
    & \multicolumn{5}{l}{.710$\pm$.003}\\
    $S^n$ 
    & .717$\pm$.007 & .488$\pm$.132 & .940$\pm$.004/.940$\pm$.004 & .626$\pm$.013 (.00$\pm$.000) & 1.00$\pm$.000\\
    $S^v$ 
    & .717$\pm$.007 & .388$\pm$.132 & .940$\pm$.004/.230$\pm$.021 & 1.00$\pm$.000 (.37$\pm$.035) & 1.00$\pm$.000\\
    \textbf{$S^a$}
    & .716$\pm$.006 & .466$\pm$.144 & .460$\pm$.016/.460$\pm$.016 & 1.00$\pm$.000 (.00$\pm$.000) & .984$\pm$.005\\
    \textbf{$S^a$-A} 
    & .728$\pm$.006 & .432$\pm$.149 & .361$\pm$.024/.361$\pm$.024 & 1.00$\pm$.000 (.00$\pm$.000) & .950$\pm$.010\\
    \textbf{$S^a$-S} 
    & .711$\pm$.005 & .551$\pm$.094 & .420$\pm$.031/.420$\pm$.031 & 1.00$\pm$.000 (.00$\pm$.000) & .966$\pm$.016\\
    \textbf{$S^a$-AS} 
    & .715$\pm$.005 & .473$\pm$.138 & .459$\pm$.016/.459$\pm$.016 & 1.00$\pm$.000 (.00$\pm$.000) & .984$\pm$.005\\

    \hline
    \textbf{$|\Models|=30$}  
    & \multicolumn{5}{l}{.710$\pm$.003} \\
    $S^n$ 
    & .718$\pm$.004 & .512$\pm$.085 & .940$\pm$.003/.940$\pm$.003 & .620$\pm$.015 (.00$\pm$.000) & 1.00$\pm$.000\\
    $S^v$ 
    & .718$\pm$.004 & .512$\pm$.085 & .940$\pm$.003/.205$\pm$.013 & 1.00$\pm$.000 (.41$\pm$.045) & 1.00$\pm$.000\\
    \textbf{$S^a$}
    & .716$\pm$.004 & .499$\pm$.088 & .456$\pm$.017/.456$\pm$.017 & 1.00$\pm$.000 (.00$\pm$.000) & .982$\pm$.006\\
    \textbf{$S^a$-A} 
    & .729$\pm$.005 & .406$\pm$.069 & .353$\pm$.023/.353$\pm$.023 & 1.00$\pm$.000 (.00$\pm$.000) & .946$\pm$.009\\
    \textbf{$S^a$-S} 
    & .712$\pm$.005 & .518$\pm$.078 & .445$\pm$.011/.445$\pm$.011 & 1.00$\pm$.000 (.00$\pm$.000) & .978$\pm$.004\\
    \textbf{$S^a$-AS}
    & .716$\pm$.004 & .500$\pm$.088 & .456$\pm$.017/.456$\pm$.017 & 1.00$\pm$.000 (.00$\pm$.000) & .982$\pm$.006\\

    \end{tabular}
    \caption{Quantitative evaluations of ensembling methods on heloc dataset.}
    \label{tab:results-std1}
\end{table*}

\begin{table*}[h]
    \centering
    \begin{tabular}{c|ccccc}
     
    & \textbf{acc.} & \textbf{simp.} & \textbf{size M}/\textbf{C} &    \textbf{$\cfx$ val. (fail)} & \textbf{mv}\\

    \hline
    &
    \multicolumn{5}{c}{compas}\\
    \hline
    \textbf{$|\Models|=10$}  
    & \multicolumn{5}{l}{.856$\pm$.001}\\

    $S^n$ 
    & .858$\pm$.001 & .464$\pm$.127 & .980$\pm$.001/.980$\pm$.001 & .572$\pm$.008 (.00$\pm$.000) & 1.00$\pm$.000\\
    $S^v$ 
    & .858$\pm$.001 & .464$\pm$.127 & .980$\pm$.001/.174$\pm$.019 & 1.00$\pm$.000 (.51$\pm$.097) & 1.00$\pm$.000\\
    
    \textbf{$S^a$}
    & .859$\pm$.004 & .463$\pm$.137 & .369$\pm$.021/.369$\pm$.021 & 1.00$\pm$.000 (.00$\pm$.000) & .994$\pm$.002\\
    \textbf{$S^a$-A} 
    & .864$\pm$.003 & .430$\pm$.126 & .295$\pm$.042/.295$\pm$.042 & 1.00$\pm$.000 (.00$\pm$.000) & .988$\pm$.003\\
    \textbf{$S^a$-S}
    & .860$\pm$.002 & .657$\pm$.099 & .306$\pm$.036/.306$\pm$.036 & 1.00$\pm$.000 (.00$\pm$.000) & .987$\pm$.004\\
    \textbf{$S^a$-AS} 
    & .860$\pm$.005 & .501$\pm$.126 & .360$\pm$.023/.360$\pm$.023 & 1.00$\pm$.000 (.00$\pm$.000) & .994$\pm$.002\\

    \hline
    \textbf{$|\Models|=20$}      
    & \multicolumn{5}{l}{.855$\pm$.001}\\
    $S^n$ 
    & .859$\pm$.002 & .538$\pm$.112 & .978$\pm$.002/.978$\pm$.002 & .544$\pm$.004 (.00$\pm$.000) & 1.00$\pm$.000\\
    $S^v$ 
    & .859$\pm$.002 & .538$\pm$.112 & .978$\pm$.002/.111$\pm$.021 & 1.00$\pm$.000 (.60$\pm$.066) & 1.00$\pm$.000\\
    
    \textbf{$S^a$}
    & .859$\pm$.004 & .514$\pm$.115 & .331$\pm$.018/.331$\pm$.018 & 1.00$\pm$.000 (.00$\pm$.000) & .992$\pm$.006\\
    \textbf{$S^a$-A} 
    & .866$\pm$.001 & .541$\pm$.140 & .235$\pm$.027/.235$\pm$.027 & 1.00$\pm$.000 (.00$\pm$.000) & .982$\pm$.008\\
    \textbf{$S^a$-S} 
    & .857$\pm$.005 & .609$\pm$.070 & .304$\pm$.030/.304$\pm$.030 & 1.00$\pm$.000 (.00$\pm$.000) & .987$\pm$.004\\
    \textbf{$S^a$-AS} 
    & .859$\pm$.004 & .555$\pm$.111 & .324$\pm$.025/.324$\pm$.025 & 1.00$\pm$.000 (.00$\pm$.000) & .990$\pm$.006\\

    \hline
    \textbf{$|\Models|=30$}  
    & \multicolumn{5}{l}{.855$\pm$.001}\\
    $S^n$ 
    & .859$\pm$.001 & .527$\pm$.094 & .976$\pm$.002/.976$\pm$.002 & .530$\pm$.005 (.00$\pm$.000) & 1.00$\pm$.000\\
    $S^v$ 
    & .859$\pm$.001 & .527$\pm$.094 & .976$\pm$.002/.087$\pm$.014 & 1.00$\pm$.000 (.56$\pm$.008) & 1.00$\pm$.000\\
   
    \textbf{$S^a$}
    & .862$\pm$.004 & .519$\pm$.097 & .308$\pm$.010/.308$\pm$.010 & 1.00$\pm$.000 (.00$\pm$.000) & .990$\pm$.005\\
    \textbf{$S^a$-A} 
    & .865$\pm$.002 & .532$\pm$.142 & .225$\pm$.004/.225$\pm$.004 & 1.00$\pm$.000 (.00$\pm$.000) & .983$\pm$.008\\
    \textbf{$S^a$-S} 
    & .861$\pm$.005 & .567$\pm$.080 & .294$\pm$.018/.294$\pm$.018 & 1.00$\pm$.000 (.00$\pm$.000) & .988$\pm$.006\\
    \textbf{$S^a$-AS}
    & .862$\pm$.004 & .543$\pm$.099 & .303$\pm$.016/.303$\pm$.016 & 1.00$\pm$.000 (.00$\pm$.000) & .988$\pm$.007\\

    \end{tabular}
    \caption{Quantitative evaluations of ensembling methods on compas dataset.}
    \label{tab:results-std2}
\end{table*}

\begin{table*}[h]
    \centering
    \begin{tabular}{c|ccccc}
     
    & \textbf{acc.} & \textbf{simp.} & \textbf{size M}/\textbf{C} &    \textbf{$\cfx$ val. (fail)} & \textbf{mv}\\

    \hline
    &
    \multicolumn{5}{c}{credit}\\
    \hline
    \textbf{$|\Models|=10$}  
    & \multicolumn{5}{l}{.664$\pm$.007}\\

    $S^n$ 
    & .697$\pm$.019 & .588$\pm$.152 & .817$\pm$.005/.817$\pm$.005 & .757$\pm$.013 (.00$\pm$.000) & 1.00$\pm$.000\\
    $S^v$ 
    & .697$\pm$.019 & .588$\pm$.152 & .817$\pm$.005/.457$\pm$.010 & 1.00$\pm$.000 (.44$\pm$.044) & 1.00$\pm$.000\\
    
    \textbf{$S^a$}
    & .694$\pm$.011 & .600$\pm$.147 & .580$\pm$.012/.580$\pm$.012 & 1.00$\pm$.000 (.00$\pm$.000) & .953$\pm$.001\\
    \textbf{$S^a$-A} 
    & .710$\pm$.003 & .593$\pm$.186 & .486$\pm$.044/.486$\pm$.044 & 1.00$\pm$.000 (.00$\pm$.000) & .825$\pm$.005\\
    \textbf{$S^a$-S}
    & .689$\pm$.010 & .626$\pm$.132 & .565$\pm$.020/.565$\pm$.020 & 1.00$\pm$.000 (.00$\pm$.000) & .925$\pm$.021\\
    \textbf{$S^a$-AS} 
    & .696$\pm$.016 & .607$\pm$.148 & .578$\pm$.011/.578$\pm$.011 & 1.00$\pm$.000 (.00$\pm$.000) & .946$\pm$.004\\

    \hline
    \textbf{$|\Models|=20$}      
    & \multicolumn{5}{l}{.663$\pm$.004}\\
    $S^n$ 
    & .708$\pm$.012 & .571$\pm$.006 & .810$\pm$.004/.810$\pm$.004 & .734$\pm$.010 (.00$\pm$.000) & 1.00$\pm$.000\\
    $S^v$ 
    & .708$\pm$.012 & .571$\pm$.006 & .810$\pm$.004/.351$\pm$.030 & 1.00$\pm$.000 (.62$\pm$.039) & 1.00$\pm$.000\\
    
    \textbf{$S^a$}
    & .691$\pm$.009 & .580$\pm$.006 & .557$\pm$.008/.557$\pm$.008 & 1.00$\pm$.000 (.00$\pm$.000) & .961$\pm$.001\\
    \textbf{$S^a$-A} 
    & .709$\pm$.011 & .586$\pm$.007 & .481$\pm$.046/.481$\pm$.046 & 1.00$\pm$.000 (.00$\pm$.000) & .862$\pm$.005\\
    \textbf{$S^a$-S} 
    & .684$\pm$.013 & .590$\pm$.006 & .549$\pm$.012/.549$\pm$.012 & 1.00$\pm$.000 (.00$\pm$.000) & .947$\pm$.010\\
    \textbf{$S^a$-AS} 
    & .693$\pm$.008 & .581$\pm$.006 & .556$\pm$.009/.556$\pm$.009 & 1.00$\pm$.000 (.00$\pm$.000) & .959$\pm$.012\\

    \hline
    \textbf{$|\Models|=30$}  
    & \multicolumn{5}{l}{.663$\pm$.004}\\
    $S^n$ 
    & .710$\pm$.006 & .540$\pm$.040 & .807$\pm$.004/.807$\pm$.004 & .727$\pm$.007 (.00$\pm$.000) & 1.00$\pm$.000\\
    $S^v$ 
    & .710$\pm$.006 & .540$\pm$.040 & .807$\pm$.004/.311$\pm$.045 & 1.00$\pm$.000 (.70$\pm$.057) & 1.00$\pm$.000\\
   
    \textbf{$S^a$}
    & .683$\pm$.007 & .549$\pm$.039 & .551$\pm$.005/.551$\pm$.005 & 1.00$\pm$.000 (.00$\pm$.000) & .943$\pm$.022\\
    \textbf{$S^a$-A} 
    & .711$\pm$.007 & .552$\pm$.048 & .441$\pm$.051/.441$\pm$.051 & 1.00$\pm$.000 (.00$\pm$.000) & .850$\pm$.059\\
    \textbf{$S^a$-S} 
    & .684$\pm$.013 & .555$\pm$.034 & .546$\pm$.007/.546$\pm$.007 & 1.00$\pm$.000 (.00$\pm$.000) & .940$\pm$.022\\
    \textbf{$S^a$-AS}
    & .683$\pm$.007 & .549$\pm$.039 & .551$\pm$.005/.551$\pm$.005 & 1.00$\pm$.000 (.00$\pm$.000) & .944$\pm$.023\\

    \end{tabular}
    \caption{Quantitative evaluations of ensembling methods on credit dataset.}
    \label{tab:results-std3}
\end{table*}

\end{document}